\def\year{2019}\relax
\documentclass[letterpaper]{article} 
\usepackage{aaai19}  
\usepackage{times}  
\usepackage{helvet}  
\usepackage{courier}  
\usepackage{url}  
\usepackage{graphicx}  

\usepackage{booktabs}       
\usepackage{amsfonts}       
\usepackage{nicefrac}       
\usepackage{microtype,amsmath,amsthm,subcaption,enumitem}      

\newcommand{\bd}{\boldsymbol}
\newcommand{\mb}{\mathbf}
\newcommand{\be}{\begin{equation}}
\newcommand{\ee}{\end{equation}}

\newtheorem{thm}{Theorem}
\newtheorem{dfn}{Definition}

 \newcommand{\citet}[1]
  {\citeauthor{#1} ̃\shortcite{#1}}

\begin{document}
%
\title{Sensitivity Analysis of Deep Neural Networks}
\author{Hai Shu\\
Department of Biostatistics\\
The University of Texas MD Anderson Cancer Center\\
 Houston, Texas, USA
\And Hongtu Zhu \\
AI Labs, Didi Chuxing\\
 Beijing, China\\
zhuhongtu@didiglobal.com}

\maketitle
\begin{abstract}
Deep   neural networks (DNNs)   have achieved superior performance in various  prediction tasks, but 
can be very vulnerable to adversarial examples or perturbations.  
Therefore, it is crucial to 
measure   
the sensitivity
of DNNs to various forms of perturbations  in real applications.  
We introduce a novel perturbation manifold and its associated influence measure to quantify  the effects of various perturbations 
on DNN classifiers. Such 
perturbations include various  external and internal perturbations to input samples and network parameters. 
The proposed measure is motivated by  information geometry and  provides
desirable invariance properties.  
We demonstrate that our influence measure is useful for  four model building tasks: 
detecting potential `outliers',  analyzing the sensitivity of
model architectures, 
comparing network sensitivity  between training and test sets, 
and locating vulnerable areas.
Experiments show reasonably good performance of the proposed measure
for the popular DNN models ResNet50 and DenseNet121 on CIFAR10 and MNIST datasets.
\end{abstract}

\section{Introduction}
Deep neural networks (DNNs) have exhibited impressive power in image classification and outperformed human detection in the ImageNet challenge \cite{Russ15,He15,He16,Huan17}.  Despite this huge success, it is well known  that 
state-of-the-art DNNs  can be  sensitive to small perturbations \cite{Szeg13,Good15,Moos16,Carl17,Su17}. This vulnerability has called into question their usage in safety-critical applications, including self-driving cars~\cite{Boja16} and face recognition~\cite{Shar17}, among many others~\cite{Akht18}.
There is rich literature on quantifying the sensitivity or robustness  of DNNs to external perturbations that affect the input samples; see \cite{Fawz17,Akht18,Nova18}.
For instance, one popular robustness measure computes the minimum adversarial distortion for a given sample \cite{Moos16,Hein17,Weng18}.
However, very little work has been done on measuring the effects of various internal perturbations to network trainable parameters on DNNs. 
To the best of our knowledge,  \cite{Cheney17} is the first paper to examine
the robustness of AlexNet \cite{Kriz12} by tracking the classification performance over several chosen standard deviations of Gaussian perturbations to network weights.

The aim of this paper is to develop a novel perturbation manifold and its associated influence measure to evalute the effects of various perturbations to 
input samples and/or network trainable parameters. 
Our influence measure   is a novel extension of the  local influence measures proposed in \cite{Zhu07,Zhu11}  to classification problems by using information geometry \cite{Amar85,Amar00}. 
Compared with the existing methods~\cite{Akht18}, we make the following two major methodological contributions.

Our influence measure is motivated by information geometry, and its calculation  is computationally straightforward and does not require optimizing any objective function.    
When the dimension of the perturbation vector is larger than the number of classes,  
the perturbation manifold in \cite{Zhu07,Zhu11} has a singular metric tensor and thus fails to
form a Riemannian manifold.
We address this singularity issue by introducing a low-dimensional transform and show that 
our  influence measure still provides  the invariance under diffeomorphisms of the original perturbation.  
Such an invariance property is critical for assessing the simultaneous effects 
or comparing the individual impacts
of different external and/or internal perturbations within or between DNNs 
without concerning their difference in scales,
such as the comparison between perturbations to trainable parameters in a convolution layer
and those in a batch normalization layer within a single DNN.
In contrast, existing measures, such as the Jacobian norm \cite{Nova18} and Cook's local influence measure~\cite{Cook86}, 
do not have this invariance property, leading to some misleading results.

Our proposed influence measure is applicable to  various forms of external and internal perturbations and useful for four important
model building tasks:
(i) detecting potential `outliers', (ii) analyzing the sensitivity of model architectures,  (iii) comparing network sensitivity  between training and test sets, 
and (iv) locating vulnerable areas. 
For task (i), 
downweighting  outliers  may be used to train a DNN with increased robustness.  
Task (ii) may serve  as a guide to the improvement of an existing network architecture. 
Task (iii) can evaluate the heterogeneity of the model robustness between training and test sets, and combining tasks (i)--(iii) may be useful for selecting  DNNs.
For task (iv), the discovered vulnerable locations in a given image 
can be utilized to either craft adversarial examples or fortify a DNN with data augmentation. 
The application of our influence measure to tasks (i)-(iv) 
is illustrated for two popular DNNs, ResNet50 \cite{He16} and DenseNet121 \cite{Huan17},
on the benchmark datasets CIFAR10 and MNIST.

\section{Method}	

\subsection{Perturbation Manifold}
Given an input image $\bd{x}$ and a DNN model $N$ with a trainable parameter vector $\bd{\theta}$,
the prediction probability for the response variable $y\in \{1,\dots,K\}$ is denoted as 
$
P(y|\bd{x},\bd{\theta})=N_{\bd{\theta}}(y,\bd{x}).
$
Let $\bd{\omega}=(\omega_1,\dots,\omega_p)^T$ be a perturbation vector varying in an open subset $\Omega \subseteq \mathbb{R}^p$. 
The perturbation $\bd{\omega}$ can be flexibly imposed on $\bd{x}$, $\bd{\theta}$, or even the combination of $\bd{x}$ and $\bd{\theta}$.
Denote $P(y|\bd{x},\bd{\theta}, \bd{\omega})$ to be the prediction probability  under perturbation $\bd{\omega}$ such that $\sum_{y=1}^K P(y|\bd{x},\bd{\theta}, \bd{\omega})=1$. 
It is assumed that there is a $\bd{\omega}_0\in \Omega$ such that   $P(y|\bd{x},\bd{\theta}, \bd{\omega}_0)=P(y|\bd{x},\bd{\theta})$.
Also,  $\{P(y|\bd{x},\bd{\theta}, \bd{\omega})\}_{y=1}^K$ is assumed to be 
positive and sufficiently smooth for all $\bd{\omega}\in\Omega$.

Following the development in   \cite{Zhu07,Zhu11}, we    may define $\mathcal{M}=\{P(y|\bd{x},\bd{\theta}, \bd{\omega}):\bd{\omega}\in \Omega\}$ as a perturbation manifold. 
The tangent space of $\mathcal{M}$ at $\bd{\omega}$ is denoted by $T_{\bd{\omega}}$,
which is spanned by 
$\{\partial \ell(\bd{\omega}|y,\bd{x},\bd{\theta})/\partial \omega_i\}_{i=1}^p$,
where $\ell(\bd{\omega}|y,\bd{x},\bd{\theta})=\log P(y|\bd{x},\bd{\theta}, \bd{\omega})$.
Let	$
\mb{G}_{\bd{\omega}}(\bd{\omega})=\sum_{y=1}^K \partial_{\bd{\omega}}^T \ell(\bd{\omega}|y,\bd{x},\bd{\theta})\partial_{\bd{\omega}} \ell(\bd{\omega}|y,\bd{x},\bd{\theta})   P(y|\bd{x},\bd{\theta},\bd{\omega})
$~with 
$\partial_{\bd{\omega}}=(\partial/\partial \omega_1,\dots,\partial/\partial \omega_p)$.
If 	$\mb{G}_{\bd{\omega}}(\bd{\omega})$ is positive definite, 
then for any two tangent vectors $v_i(\bd{\omega})=\bd{h}_i^T\partial_{\bd{\omega}}^T \ell(\bd{\omega}|y,\bd{x},\bd{\theta})      \in T_{\bd{\omega}}$, $i=1,2$, 
where $\bd{h}_i^T$ denotes the coordinate vector of $v_i(\bd{\omega})$
on the basis $\partial_{\bd{\omega}} \ell(\bd{\omega}|y,\bd{x},\bd{\theta})$,
the inner product can be defined by	
\begin{align}\label{inner product}
\langle  v_1(\bd{\omega}),v_2(\bd{\omega})   \rangle
&= \sum_{y=1}^K   v_1(\bd{\omega})v_2(\bd{\omega})    P(y|\bd{x},\bd{\theta},\bd{\omega})\nonumber\\
&=\bd{h}_1^T \mb{G}_{\bd{\omega}}(\bd{\omega})\bd{h}_2.
\end{align}
Subsequently, the length of $v_1(\bd{\omega})$ is given by 
\[
\| v_1(\bd{\omega}) \|= \sqrt{\langle  v_1(\bd{\omega}),v_1(\bd{\omega})   \rangle}
=\left[ \bd{h}_1^T  \mb{G}_{\bd{\omega}}(\bd{\omega})  \bd{h}_1    \right]^{1/2}.
\]
With the above inner product defined by $\mb{G}_{\bd{\omega}}(\bd{\omega})$, $\mathcal{M}$  is a Riemannian manifold and $\mb{G}_{\bd{\omega}}(\bd{\omega})$ is  the Riemannian metric tensor  \cite{Amar85,Amar00}.

We need the positive definiteness of $\mb{G}_{\bd{\omega}}(\bd{\omega})$. 
However, 
$\mb{G}_{\bd{\omega}}(\bd{\omega})$ as a sum of $K$ \mbox{rank-1} matrices
has $\text{rank}(\mb{G}_{\bd{\omega}}(\bd{\omega}))\le K$, so it  is a singular matrix when $K<p$. The case $K<p$  is  true in many classification problems since   the number of classes is often much smaller than the dimension of $\bd{\omega}$. 
The singularity of $\mb{G}_{\bd{\omega}}(\bd{\omega})$ indicates
that the $p$ tangent vectors $\partial \ell(\bd{\omega}|y,\bd{x},\bd{\theta})/\partial \omega_i$
are linearly dependent
and thus some components of $\bd{\omega}$ 
are redundant.
In addition,  our focus is on the small perturbations around $\bd{\omega}_0$.
We  hence transform the $p$-dimensional $\bd{\omega}$ to be a vector $\bd{\nu}$ such that  $\mb{G}_{\bd{\nu}}(\bd{\nu})$ is positive definite 
in a small neighborhood of $\bd{\nu}_0$ that corresponds to $\bd{\omega}_0$. 

Our low-dimensional transform is implemented as follows.
We first obtain a compact singular value decomposition (cSVD)
of $\mb{G}_{\bd{\omega}}(\bd{\omega}_0)$.
For very large $p$,
rather than the direct but extremely expensive cSVD computation of the $p\times p$ matrix,
we apply a computationally efficient approach using the cSVD of the much smaller $p\times K$ matrix 
$\mb{L}_0=[\partial_{\bd{\omega}}^T \ell(\bd{\omega}_0|y,\bd{x},\bd{\theta})P^{1/2}(y|\bd{x},\bd{\theta},\bd{\omega}_0)]_{1\le y\le K}$
by noticing that $\mb{G}_{\bd{\omega}}(\bd{\omega}_0)=\mb{L}_0\mb{L}_0^T$.  
Let $r_0=\text{rank}(\mb{G}_{\bd{\omega}}  (\bd{\omega}_0)  )$.
The usual cSVD computation can easily yield that $\mb{L}_0=\mb{B}_0 \mb{A}_0$
and $\mb{A}_0\mb{A}_0^T=\mb{U}_A\mb{\Lambda}_0\mb{U}_A^T$,
where $\mb{B}_0$ is a $p\times r_0$ matrix with orthonormal columns,
$\mb{U}_A$ is a $r_0\times r_0$ orthogonal matrix and $\mb{\Lambda}_0$ is a $r_0\times r_0$ diagonal matrix.
We hence obtain the cSVD: 
$\mb{G}_{\bd{\omega}}(\bd{\omega}_0)=\mb{U}_0 \mb{\Lambda}_0 \mb{U}_0^T$
with $\mb{U}_0=\mb{B}_0\mb{U}_A$.
Define the desirable transform of $\bd{\omega}\in \Omega$ by
$\bd{\nu}=\mb{\Lambda}_0^{1/2}\mb{U}_0^T \bd{\omega}$.
Denote 
$P(y|\bd{x},\bd{\theta},\bd{\nu})=P(y|\bd{x},\bd{\theta},\bd{\omega}=\mb{U}_0 \mb{\Lambda}_0^{-1/2}\bd{\nu}+\bd{\xi}_0)$,
where $\bd{\xi}_0=\bd{\omega}_0-\mb{U}_0 \mb{\Lambda}_0^{-1/2}\bd{\nu}_0$
and $\bd{\nu}_0=\mb{\Lambda}_0^{1/2}\mb{U}_0^T \bd{\omega}_0$.
It follows
from $\partial_{\bd{\nu}}\ell
=\partial_{\bd{\omega}}\ell\mb{U}_0\mb{\Lambda}_0^{-1/2}$ 
that $\mb{G}_{\bd{\nu}}(\bd{\nu}_0)= \mb{I}$.
By the smoothness of $P(y|\bd{x},\bd{\theta},\bd{\omega})$ in $\bd{\omega}\in \Omega$,
the metric tensor $\mb{G}_{\bd{\nu}}(\bd{\nu})$ is positive definite in
an open ball $B_{\bd{\nu}_0}$ centered at $\bd{\nu}_0$.  

\begin{dfn} 
	We define the Riemannian manifold $\mathcal{M}_{\bd{\nu}_0}=
	\{ P(y|\bd{x},\bd{\theta},\bd{\nu}):  \bd{\nu}\in      B_{\bd{\nu}_0}\}$
	with the inner product defined by $\mb{G}_{\bd{\nu}}(\bd{\nu})$  in \eqref{inner product}
	as  the  perturbation manifold around $\bd{\nu}_0$.
\end{dfn}

\subsection{Influence Measure}
Let $f(\bd{\omega})$ be the objective function of interest for sensitivity analysis.
We  define the influence measure to evaluate the discrepancy
of the objective function $f(\bd{\omega})$ at two points, $\bd{\omega}_1$ and $\bd{\omega}_2$, corresponding to 
$\bd{\nu}_i=\mb{\Lambda}_0^{1/2}\mb{U}_0^T \bd{\omega}_i$, $i=1,2$
on the perturbation manifold $\mathcal{M}_{\bd{\nu}_0}$.
Let $C(t)=P(y|\bd{x},\bd{\theta},\bd{{\nu}}(t))$ be a smooth curve
on $\mathcal{M}_{\bd{\nu}_0}$ 
connecting $\bd{\nu}_1=\bd{\nu}(t_1)$ to  $\bd{\nu}_2=\bd{\nu}(t_2)$,
where $\bd{\nu}(t)=\mb{\Lambda}_0^{1/2}\mb{U}_0^T \bd{\omega}(t)$
with a smooth curve $\bd{\omega}(t)$
connecting $\bd{\omega}_1=\bd{\omega}(t_1)$
to $\bd{\omega}_2=\bd{\omega}(t_2)$.
The distance between $\bd{\nu}_1$ and $\bd{\nu}_2$ along the curve $C(t)$ is defined by
\[
S_{C}(\bd{{\nu}}_1,\bd{{\nu}}_2)
=\int_{t_1}^{t_2} \left[\dot{\bd{\nu}}(t)^T \mb{G}_{\bd{\nu}}(\bd{\nu}(t)) \dot{\bd{\nu}}(t) \right]^{1/2} dt,
\]
with $\dot{\bd{\nu}}(t)=d\bd{\nu}(t)/dt $.
Following \cite{Zhu11},
the influence measure for $f(\bd{\omega})$ along $C(t)$
is  given by 
\[
\text{I}_{C}(\bd{{\omega}}_1,\bd{{\omega}}_2)=\frac{[f(\bd{{\omega}}_1)-f(\bd{{\omega}}_2)]^2}{S_{C}^2(\bd{{\nu}}_1,\bd{{\nu}}_2)}.
\]
Let $\bd{\omega}(0)=\bd{\omega}_0$, then $\bd{\nu}(0)=\bd{\nu}_0$.
We define the (first-order) local influence measure of  $f(\bd{\omega})$ at $\bd{\omega}_0$ by
\be\label{FI}
\text{FI}_{\bd{\omega}}(\bd{\omega}_0)=\max_{C}  \lim_{t\to 0} \text{I}_{C}(\bd{{\omega}}(t),\bd{{\omega}}(0)).
\ee
Denote $\bd{h}_{\nu}=\dot{\bd{\nu}}(0)=\mb{\Lambda}_0^{1/2}\mb{U}_0^T\bd{h}_{\omega}$,
$\bd{h}_{\omega}=\dot{\bd{\omega}}(0)$,
$\dot{\bd{\omega}}(t)=d\bd{\omega}/dt$,
 $\nabla_{f(\bd{\omega}_0)}=\partial_{\bd{\omega}} f|_{\bd{\omega}=\bd{\omega}_0}$,  and $\mb{H}_{f(\bd{\omega}_0)}=\frac{\partial^2 f}{\partial\bd{\omega}\partial\bd{\omega}^T}\big|_{\bd{\bd{\omega}}=\bd{\omega}_0}$.
Plugging
$S^2_{C}(\bd{\nu}(t),\bd{\nu}(0))=t^2 \bd{h}_{\nu}^T\mb{G}_{\bd{\nu}}(\bd{\nu}_0) \bd{h}_{\nu}+o(t^2)$
and
$
f(\bd{\omega}(t))=f(\bd{\omega}(0))+ \nabla_{f(\bd{\omega}_0)}\bd{h}_{\omega}t+\frac{1}{2}(\bd{h}_{\omega}^T \mb{H}_{f(\bd{\omega}_0)} \bd{h}_{\omega}+\nabla_{f(\bd{\omega}_0)} \frac{d^2\bd{\omega}(0)}{dt^2})t^2+o(t^2)
$
into   \eqref{FI} yields the closed form
\begin{align}\label{FI closed form}
\text{FI}_{\bd{\omega}}(\bd{\omega}_0)
&=\max_{\bd{h}_{\nu}}  \frac{ \bd{h}_{\nu}^T \nabla_{f(\bd{\nu}_0)}^T    \nabla_{f(\bd{\nu}_0)}\bd{h}_{\nu}}{\bd{h}_{\nu}^T\mb{G}_{\bd{\nu}}(\bd{\nu}_0)\bd{h}_{\nu}}\nonumber\\
&=  \nabla_{f(\bd{\nu}_0)}   \nabla_{f(\bd{\nu}_0)}^T\\
&=  \nabla_{f(\bd{\omega}_0)}  \mb{G}^\dag_{\bd{\omega}}(\bd{\omega}_0)
\nabla_{f(\bd{\omega}_0)}^T,\nonumber
\end{align}
where
$\nabla_{f(\bd{\nu}_0)} :=\nabla_{f(\bd{\omega}_0)}\mb{U}_0 \mb{\Lambda}_0^{-1/2} $,
$\mb{G}^\dag_{\bd{\omega}}(\bd{\omega}_0)$
is the pseudoinverse  of $\mb{G}_{\bd{\omega}}(\bd{\omega}_0)$,
and we used
the identities
$\mb{G}_{\bd{\nu}}(\bd{\nu}_0)=\mb{I}$
and
$
\partial_{\bd{\omega}}f\partial_{t}\bd{\omega}
=\partial_{\bd{\omega}}f
\partial_{\bd{\nu}}\bd{\omega}
\partial_{t}\bd{\nu}
=\partial_{\bd{\omega}}f
\partial_{\bd{\nu}}\bd{\omega}
\partial_{\bd{\omega}}\bd{\nu}
\mb{U}_0\mb{\Lambda}_0^{-1/2}
\partial_{t}\bd{\nu}
=\partial_{\bd{\omega}}f
\mb{U}_0\mb{\Lambda}_0^{-1/2}
\partial_{t}\bd{\nu}.
$

\begin{dfn}
	We define the influence measure of $f(\bd{\omega})$ at $\bd{\omega}_0$ by $\textup{FI}_{\bd{\omega}}(\bd{\omega}_0)$ 
	given in \eqref{FI} with the closed form
in \eqref{FI closed form}.
\end{dfn}

\begin{thm}\label{thm1}	
	Suppose that $\bd{\varphi}$ is a diffeomorphism of $\bd{\omega}$.
	Then, $\textup{FI}_{\bd{\omega}}(\bd{\omega}_0)$ is invariant with respect to any reparameterization corresponding to $\bd{\varphi}$.	
\end{thm}

\begin{proof}	
	Let  $\bd{\varphi}=\bd{\varphi}(\bd{\omega})$, $\bd{\omega}=\bd{\omega}(\bd{\varphi})$, and $\bd{\varphi}_0=\bd{\varphi}(\bd{\omega}_0)$.
	Denote their Jacobian matrices by $\mb{\Phi}=\partial  \bd{\varphi}/\partial \bd{\omega}$ and $\mb{\Psi}=\partial \bd{\omega}   /\partial  \bd{\varphi}$.
	Differentiating 
	$\bd{\omega}=\bd{\omega}(\bd{\varphi}(\bd{\omega}))$
	with respect to $\bd{\omega}$ yields $\mb{I}=\mb{\Psi}\mb{\Phi}$.
	Denote $\mb{\Psi}_0=\mb{\Psi}(\bd{\varphi}_0)$, $\mb{\Phi}_0=\mb{\Phi}(\bd{\omega}_0)$,
	 $\dot{\bd{\omega}}_0=\dot{\bd{\omega}}(0)$
	and $\dot{\bd{\varphi}}_0=d{\bd{\varphi}}(0)/dt$.
	We have
	\begin{align*}
	\text{FI}_{\bd{\omega}}(\bd{\omega}_0)
	&=\max_{\bd{h}_{\omega}} \frac{\bd{h}_\omega^T\nabla_{f(\bd{\omega}_0)}^T   \nabla_{f(\bd{\omega}_0)}
	 \bd{h}_\omega
	}{\bd{h}_\omega^T \mb{U}_0  \mb{\Lambda}_0   \mb{U}_0^T \bd{h}_\omega}\\
	&=\max_{\dot{\bd{\omega}}_0} \frac{{\dot{\bd{\omega}}_0}^T \nabla_{f(\bd{\omega}_0)}^T   \nabla_{f(\bd{\omega}_0)}\dot{\bd{\omega}}_0
	}{{\dot{\bd{\omega}}_0}^T \mb{G}_{\bd{\omega}}(\bd{\omega}_0) {\dot{\bd{\omega}}_0}}\\
	&=   \max_{\dot{\bd{\omega}}_0} 
	\frac{
		{\dot{\bd{\omega}}_0}^T
		\mb{\Phi}_0^T \mb{\Psi}_0^T  \nabla_{f(\bd{\omega}_0)}^T   \nabla_{f(\bd{\omega}_0)} \mb{\Psi}_0  \mb{\Phi}_0  \dot{\bd{\omega}}_0
	}
	{
		{\dot{\bd{\omega}}_0}^T \mb{\Phi}_0^T \mb{\Psi}_0^T 
		\mb{G}_{\bd{\omega}}(\bd{\omega}_0) \mb{\Psi}_0  \mb{\Phi}_0\dot{\bd{\omega}}_0
	}  \\
	&=
	\max_{\dot{\bd{\varphi}}_0} \frac{{\dot{\bd{\varphi}}_0}^T \nabla_{f(\bd{\varphi}_0)}^T   \nabla_{f(\bd{\varphi}_0)}\dot{\bd{\varphi}}_0
	}{{\dot{\bd{\varphi}}_0}^T \mb{G}_{\bd{\varphi}}(\bd{\varphi}_0) \dot{\bd{\varphi}}_0}
	=\text{FI}_{\bd{\varphi}}(\bd{\varphi}_0).
	\end{align*}
\end{proof}

Theorem~\ref{thm1} shows the invariance of 
$\textup{FI}_{\bd{\omega}}(\bd{\omega}_0)$
under any diffeomorphic (e.g., scaling) reparameterization of 
the original perturbation vector $\bd{\omega}$ rather than $\bd{\nu}$.
This result is analogous to those in \cite{Zhu07,Zhu11}, but we extend it to
cases where the original perturbation model $\mathcal{M}$ 
with $\mb{G}_{\bd{\omega}}(\bd{\omega})$ is not a Riemannian manifold,
especially when $K<p$.

The invariance property is not enjoyed by the widely used Jacobian norm \cite{Nova18}
and Cook's local influence measure~\cite{Cook86}.
For example, consider the perturbation $\bd{\alpha}+\Delta\bd{\alpha}$, where $\bd{\alpha}=(\alpha_1,\dots,\alpha_p)^T$
is a subvector of $(\bd{x}^T,\bd{\theta}^T)^T$, and the scaling version $\bd{\alpha}'+\Delta\bd{\alpha}'$ with $\bd{\alpha}'=k\bd{\alpha}$.
Let $(\bd{\omega},\bd{\omega}_0)=(\Delta\bd{\alpha},\bd{0})$ and its scaling counterpart $(\bd{\omega}',\bd{\omega}'_0)=(\Delta\bd{\alpha}',\bd{0})$. We have that
the Jacobian norm 
\be\label{Jacob norm}
\| \mb{J}(\bd{\alpha}) \|_F=\Big[\sum_{i=1}^p\Big(\frac{\partial f}{\partial {\alpha}_i}\Big |_{\bd{\omega}=\bd{\omega}_0}\Big)^2\Big]^{1/2}
=k\| \mb{J}(\bd{\alpha}') \|_F
\ee
and the Cook's local influence measure
\begin{align}\label{Cook influ}
C_{\bd{\eta},\bd{\omega}}
&=\frac{1}{(1+\nabla_{f(\bd{\omega}_0)}   \nabla_{f(\bd{\omega}_0)}^T  )^{1/2}    }
\frac{\bd{\eta}^T  \mb{H}_{f(\bd{\omega}_0)} \bd{\eta}    }{\bd{\eta}^T (\mb{I}+   \nabla_{f(\bd{\omega}_0)}^T  \nabla_{f(\bd{\omega}_0)} )  \bd{\eta} }\nonumber\\
&\ne C_{\bd{\eta},\bd{\omega}'}=C_{k\bd{\eta},\bd{\omega}'}
\end{align}
with $\bd{\omega}(t)=\bd{\omega}_0+t\bd{\eta}$ are not scaling-invariant.
This is problematic especially when the scale heterogeneity exists between parameters to which the perturbations are imposed.
For instance, in the simultaneous perturbations to both input image $\bd{x}$ and trainable network parameters $\bd{\theta}$, i.e., $\bd{\alpha}=(\bd{x}^T,\bd{\theta}^T)^T$, the contribution of $\Delta\bd{x}$ appears to be exaggerated if $\bd{x}$ is scaled with larger values than $\bd{\theta}$.
Another example is the comparison between perturbations to trainable parameters (weights and bias) in a convolution layer and those (shift/scale parameters) in a batch normalization layer.
There are no uniform criteria for the scaling because either rescaling to a unit norm or keeping on the original scales seems to have its own advantages. 
However, our influence measure evades this scaling issue by utilizing the metric tensor of the perturbation manifold rather than that of the usual Euclidean space.

\begin{figure*}[t]
		\centering
	\includegraphics[width=\linewidth]{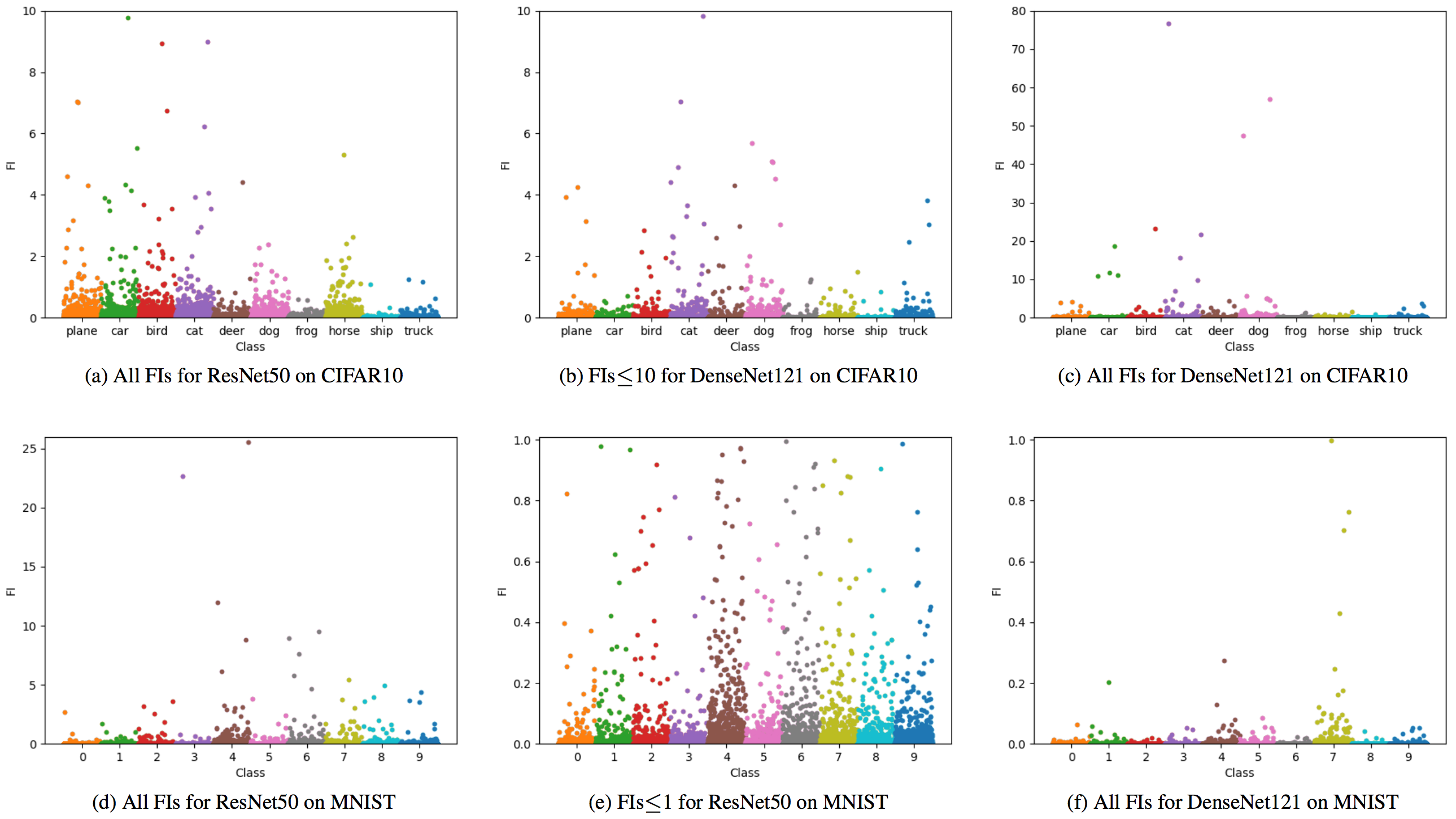}	
	\caption{Manhattan plots for Setup 1.}
	\label{setup1 fig}
\end{figure*}

\subsection{Perturbation Examples}\label{perturb example}
In this subsection, we illustrate how to compute the proposed influence measure for a trained DNN model $P(y|\bd{x},\bd{\theta})=N_{\bd{\theta}}(y,\bd{x})$. We consider the following commonly used perturbations to the input image $\bd{x}$ or the trainable parameters $\bd{\theta}=(\bd{\theta}_1^T,\dots,\bd{\theta}_{L}^T)^T$, where 
$\bd{\theta}_l$ are the parameters in the $l$-th trainable network layer.
\begin{itemize}
\item 	Case 1: $\bd{x}+\Delta \bd{x}$;

\item Case 2: $\bd{\theta}+\Delta \bd{\theta}$;

\item Case 3: $\bd{\theta}_l+\Delta \bd{\theta}_l$.

\end{itemize}

All three cases can be written in a unified form
$\bd{\alpha}+\Delta \bd{\alpha}$ with $\bd{\alpha}\in \{\bd{x},\bd{\theta},\bd{\theta}_l\}$.
Let the perturbation vector $\bd{\omega}=\Delta \bd{\alpha}$ and $\bd{\omega}_0=\bd{0}$.
For the influence measure $\text{FI}_{\bd{\omega}}(\bd{\omega}_0)$ in \eqref{FI closed form},
we have
\be\label{df/da}
\nabla_{f(\bd{\nu}_0)}= (\partial_{\bd{\alpha}} f|_{\bd{\omega}=\bd{\omega}_0}) \mb{U}_0 \mb{\Lambda}_0^{-1/2},
\ee
where $\mb{\Lambda}_0$ and $\mb{U}_0$ are obtained starting from
matrix $\mb{L}_0=[\partial_{\bd{\omega}}^T \ell(\bd{\omega}_0|y,\bd{x},\bd{\theta})P^{1/2}(y|\bd{x},\bd{\theta},\bd{\omega}_0)]_{1\le y\le K}$
through $\mb{L}_0=\mb{B}_0\mb{A}_0$,
$\mb{A}_0\mb{A}_0^T=\mb{U}_A\mb{\Lambda}_0\mb{U}_A^T$
and
$\mb{U}_0=\mb{B}_0\mb{U}_A$.
The component $\partial_{\bd{\omega}} \ell(\bd{\omega}_0|y,\bd{x},\bd{\theta})$
in $\mb{L}_0$ is now computed by
\be\label{dl/dw}
\partial_{\bd{\omega}} \ell(\bd{\omega}_0|y,\bd{x},\bd{\theta})
=\partial_{\bd{\alpha}} \log P(y|\bd{x},\bd{\theta}).
\ee
The gradients $\partial_{\bd{\alpha}} f|_{\bd{\omega}=\bd{\omega}_0}$ and $\partial_{\bd{\alpha}} \log P(y|\bd{x},\bd{\theta})$
can be calculated easily via backpropagation~\cite{goodfellow2016deep} in deep learning libraries
like TensorFlow \cite{Abad16} and Pytorch \cite{Pasz17}.

Next, we consider a specific DNN example under Case~3.
Consider the following feedforward DNN architecture before the softmax layer:
\[
\bd{n}_{\bd{\theta}}(\bd{x})=\bd{\sigma}_{L}(\mb{\Theta}_L \bd{\sigma}_{L-1}( \cdots (\mb{\Theta}_3 \bd{\sigma}_2( \mb{\Theta}_2 \bd{\sigma}_1(\mb{\Theta}_1\bd{x} )     )       ) ) ) \in  \mathbb{R}^K,
\]
where $\bd{x}\in \mathbb{R}^{k_0}$, $\mb{\Theta}_l \in \mathbb{R}^{k_l\times k_{l-1}}$, $\bd{\theta}_l=\text{vec}(\mb{\Theta}_l^T)$, and $\bd{\sigma}_l$'s are entry-wise activation functions.
For notational simplicity,  we set all bias terms to zero and consider the sigmoid function
\[
\sigma(x)=[1+\exp(-x)]^{-1}\ \text{with}\  \dot{\sigma}(x)=\sigma(x)(1-\sigma(x))
\]
for all activation functions. 
Let $\bd{i}_l(\bd{x},\bd{\theta})$ and $\bd{o}_l(\bd{x},\bd{\theta})$
be the input and output vectors of the $l$-th layer such that
$\bd{o}_l(\bd{x},\bd{\theta})=\bd{\sigma}_{l}(\bd{i}_l(\bd{x},\bd{\theta}))$
and $\bd{o}_0(\bd{x},\bd{\theta})=\bd{x}$.
The softmax function is given by
\[
\bd{g}(\bd{z})=\left( \frac{\exp(z_1)}{\sum_{k=1}^K  \exp(z_k)} ,\dots,  \frac{\exp(z_K)}{\sum_{k=1}^K  \exp(z_k)}  \right)^T.
\]
The whole DNN model is
$$P(y|\bd{x},\bd{\theta})=N_{\bd{\theta}}(y, \bd{x})=\bd{g}(\bd{n}_{\bd{\theta}}(\bd{x}))_{[y]},$$ 
for $y=1,\dots,K$, where $\bd{g}(\cdot)_{[y]}$ is the $y$-th entry of vector~$\bd{g}(\cdot)$.
Under Case~3, we have $\bd{\alpha}=\bd{\theta}_l$.
Choose the objective function $f$ to be the cross-entropy, i.e.,
\[
f(\bd{\alpha},\bd{\omega})=-\log P(y=y_{\text{true}} |\bd{x},\bd{\theta},\bd{\omega}).
\]
Hence, in \eqref{df/da} we have $\partial_{\bd{\alpha}}f|_{\bd{\omega}=\bd{\omega}_0}=-\partial_{\bd{\alpha}} \log P(y=y_{\text{true}} |\bd{x},\bd{\theta})$.
Then, to calculate the gradients in \eqref{df/da} and \eqref{dl/dw}, we only need to consider
$
\partial_{\bd{\alpha}} \log P(y|\bd{x},\bd{\theta})
= \partial_{\bd{\theta}_l} \log(\bd{g}(\bd{n}_{\bd{\theta}}(\bd{x}))_{[y]}).
$ 
Note that
$
\partial_{\bd{z}}\log( \bd{g}(\bd{z})_{{[y]}})= (\bd{e}_y- \bd{g}(\bd{z}))^T,
$
where $\bd{e}_y\in \mathbb{R}^K$ has 1 in the $y$-th entry and
0 in the others.
Moreover, 
$
\partial_{\bd{\theta}_l} \bd{n}_{\bd{\theta}}(\bd{x})
=\mb{D}_{L}\mb{\Theta}_L\mb{D}_{L-1}\cdots \mb{D}_{l+1} \mb{\Theta}_{l+1}\mb{D}_{l}\mb{O}_{l-1},
$
with $\mb{D}_{l}=\text{diag}(\{\dot{\sigma} (\bd{i}_l(\bd{x},\bd{\theta})_{[j]}) \}_{j=1}^{k_l})\in \mathbb{R}^{k_l\times k_l}$
and
$
\mb{O}_{l-1}=\text{diag} (\{ \bd{o}^T_{l-1} (\bd{x},\bd{\theta}) ,  \dots, \bd{o}^T_{l-1} (\bd{x},\bd{\theta})     \}     )\in \mathbb{R}^{k_l \times (k_{l-1}k_l )}.
$
Hence, for  \eqref{df/da} and \eqref{dl/dw},  we have 
\begin{align*}
\lefteqn{\partial_{\bd{\alpha}} \log P(y|\bd{x},\bd{\theta})}\\
&= (\bd{e}_y-   \bd{g}( \bd{n}_{\bd{\theta}}(\bd{x})  )   )^T \mb{D}_{L}\mb{\Theta}_L\mb{D}_{L-1}\cdots \mb{D}_{l+1} \mb{\Theta}_{l+1}\mb{D}_{l}\mb{O}_{l-1}.
\end{align*}

\section{Experiments}	

In this section, we investigate the performance of our local influence measure. 
We address the four tasks stated in 
Introduction through the following setups 
under the three perturbation cases in Section~\ref{perturb example}. 
\begin{itemize}[leftmargin=*]
	\item Setup 1: Compute each training image's FI under Case~1, with $f$ being the cross entropy, i.e., $f=-\log P(y=y_{\text{true}} |\bd{x},\bd{\theta},\bd{\omega})$.
	\item Setup 2: 
	Let $f=-\log P(y=y_{\text{true}} |\bd{x},\bd{\theta},\bd{\omega})$.
	\begin{itemize}[leftmargin=*]
		\item Setup 2.1: Compute each training image's FI under Case~2.
		
		\item Setup 2.2: Compute each trainable network layer's FI under Case~3
		for each training image.
	\end{itemize}
	
	\item Setup 3: Compute each image's FI under Case~1 for both training and test sets, where $f=-\log P(y=y_{\text{pred}}|\bd{x},\bd{\theta},\bd{\omega})$.
	
	\item Setup 4: Compute each pixel's FI under Case~1 for a given image.
	We adopt a multi-scale strategy taking into account the spatial effect. For each pixel, we set $\bd{x}$ in Case~1 to be the $k\times k$ square centered at the pixel with the scale $k\in \{1,3,5,7\}$. We use $f=-\log P(y=y_{\text{pred}}|\bd{x},\bd{\theta},\bd{\omega})$.
\end{itemize}	

For the cross-entropy like function $f$ in Setups 3 and 4, we use the predicted label $y_{\text{pred}}$ instead of the true label $y_{\text{true}}$ for the prediction purpose
rather than the training purpose in the first two setups.
In Setup 4, the scale of the pixel-level FI is analogous to the convolutional kernel size.

We conduct experiments on the two benchmark datasets CIFAR10 and MNIST using the two popular DNN models
ResNet50 \cite{He16} and DenseNet121 \cite{Huan17}. 
Originally, there are 50,000 and 60,000 training images for CIFAR10 and MNIST, respectively. 
As the validation sets, we use randomly selected 10\% of those images, with the same number for each class.
No data augmentation is used for the training process. 
Both datasets have 10,000 test images.
The prediction accuracy of our trained models is summarized  in Table~\ref{acc table}.

\begin{table}[h]
	\caption{Accuracy\,for\,models\,trained\,without\,data\,augmentation}
	\label{acc table}
	\centering
		\scalebox{0.90}{
	\begin{tabular}{cccccccc}
		\toprule
		&	\multicolumn{2}{c}{CIFAR10}    & &   \multicolumn{2}{c}{MNIST}           \\
		\cmidrule{2-3}\cmidrule{5-6}
		Model &	Training     & Test    & & Training     & Test  \\
		\midrule
		ResNet50 & 99.78\%&  88.70\% && 99.87\%  & 99.29\%  \\
		DenseNet121 & 99.87\%  & 91.16\% & & 99.998\% & 99.58\%   \\
		\bottomrule
	\end{tabular}
}
\end{table}

\subsection{Outlier Detection}\label{sec: outlier}
We study the outlier detection ability of our proposed influence measure under Setup~1.
Figure~\ref{setup1 fig} illustrates the results of Setup~1
by using  Manhattan plots. DenseNet121 generally has smaller FIs than ResNet50 for the 
two benchmark datasets, excluding several large FIs over 10 shown in Figure~\ref{setup1 fig}(c) for CIFAR10.
The images with the top 5 largest FIs are displayed in Figure~\ref{setup1 fig top5} for each case.
Most of the 20 images, especially those in MNIST, are  difficult even for human visual detection. This indicates the strong power of our influence measure in 
detecting outlier images.

\begin{figure}[b!]
	\centering
	\includegraphics[width=0.92\linewidth]{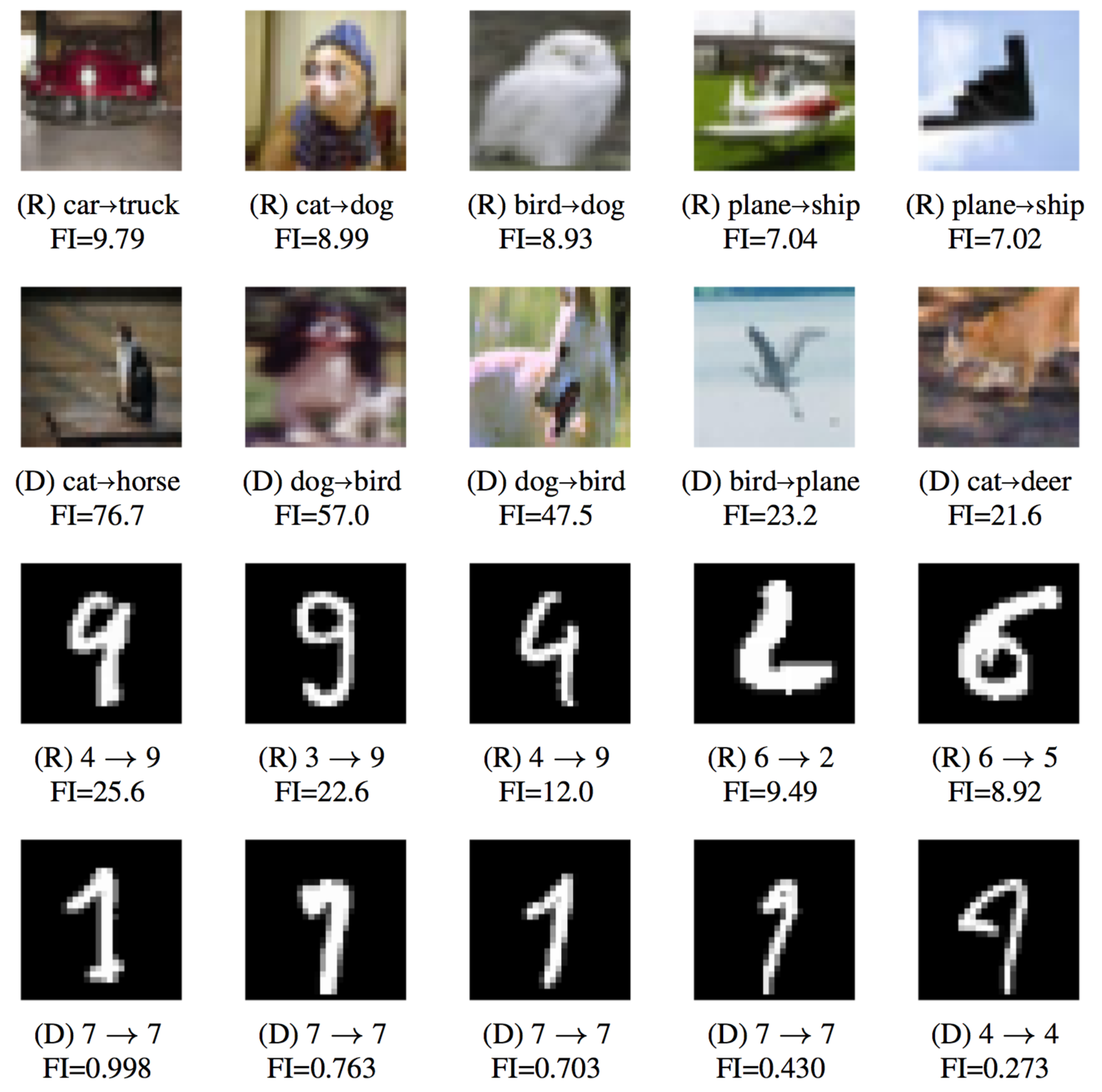}		
	\caption{Images with top 5 largest FIs in Setup\,1 for ResNet50 (R) and DensetNet121 (D). Each subcaption shows $y_{\text{true}}{\rightarrow} y_{\text{pred}}$ and FI.} 
	\label{setup1 fig top5}
\end{figure}

We further examine the  outlier detection power of our proposed influence measure by 
simulating outlier images from MNIST.
Each outlier image was generated by overlapping two training digits of different classes that are shifted up to 4 pixels in each direction, with the true label randomly set to be one of the two classes. 
The two DNN models in Table~\ref{acc table} are trained with additional 50 epochs after incorporating 2700 and 300 simulated outlier images into the training and validation sets, with accuracies reduced up to 0.38\% and 0.11\% for respective training and testing.
The original 54,000 training images are all treated as non-outliers.
We compare the proposed FI measure with the Jacobian norm given in \eqref{Jacob norm}
using the cross-entropy as the objective function $f$. 
The maximal Cook's local influence, $\max_{\bd{\eta}}C_{\bd{\eta},\bd{\omega}}$,
is not considered here due to the expensive computation of the 
very large Hessian matrix; see~\eqref{Cook influ}.
Figure~\ref{setup0 fig} shows the outlier detection results of the two considered measures.
Although the receiver operating characteristic (ROC) curves of the two measures are almost overlapping,
our FI measure significantly outperforms Jacobian norm in terms of the precision-recall (PR) curves
that are more useful for highly unbalanced data \cite{davis2006}.

\begin{figure}
		\centering
	\includegraphics[width=\linewidth]{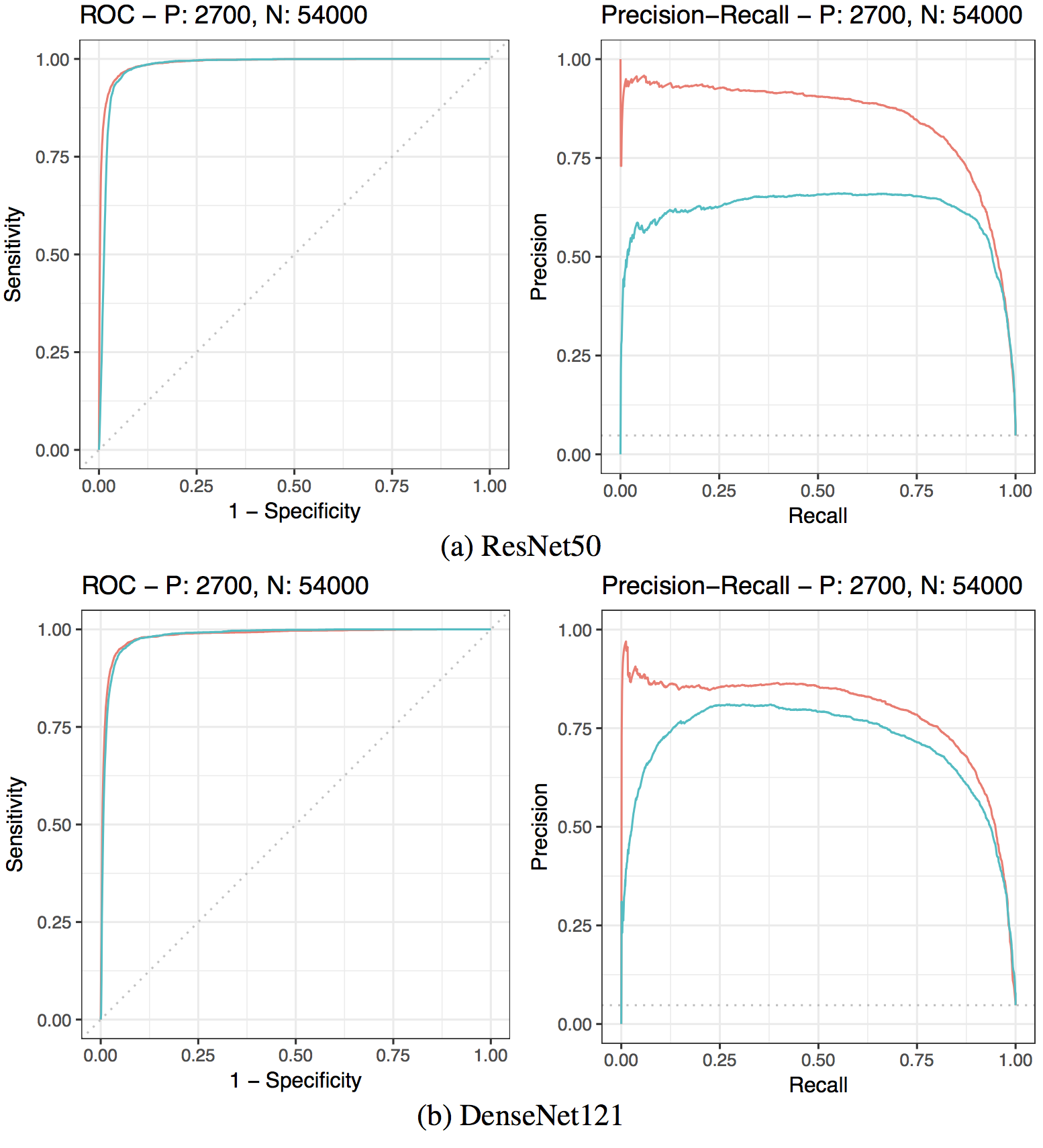}	
	\caption{ROC and PR curves of our proposed FI measure (red) and the Jacobian norm (blue) on MNIST with simulated outliers.}
	\label{setup0 fig}
\end{figure}

\begin{figure*}
		\centering
	\includegraphics[width=\linewidth]{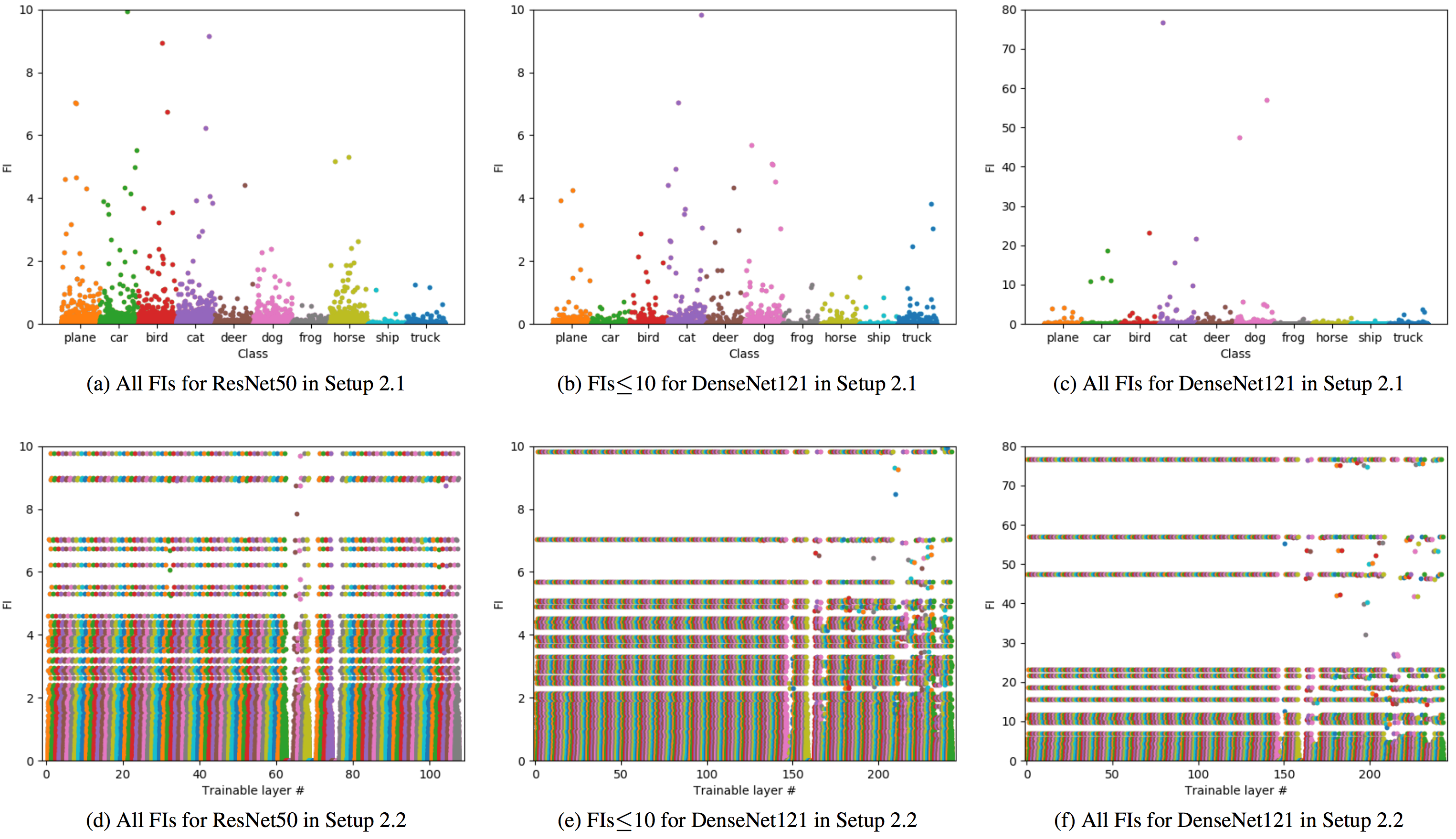}	
	\caption{Manhattan plots for Setup 2 on CIFAR10.}
	\label{setup2 CIFAR10}
\end{figure*}

\begin{figure*}[h!]
		\centering
	\includegraphics[width=\linewidth]{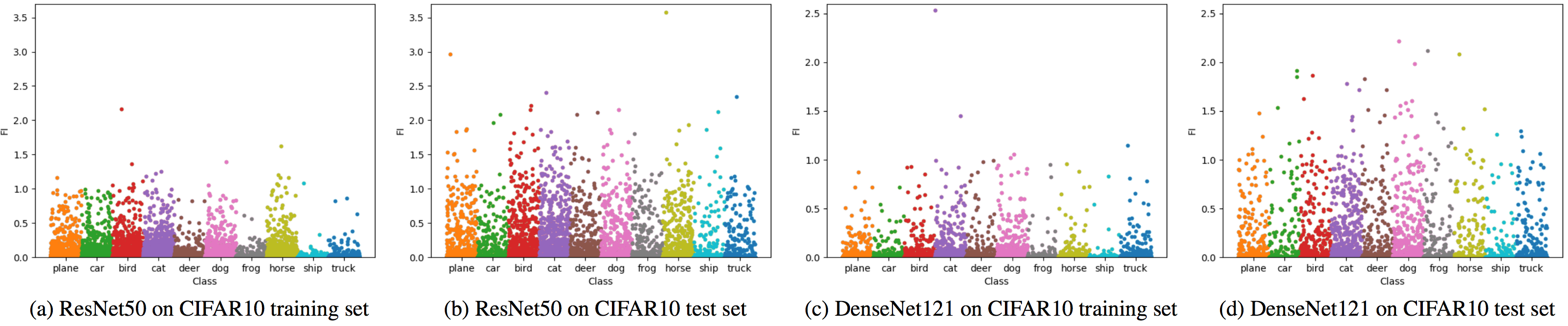}	
	\caption{Manhattan plots for Setup 3 on CIFAR10.}
	\label{setup3 CIFAR10}
\end{figure*}

\subsection{Sensitivity Analysis on DNN Architectures}\label{sec: sen archi}
We conduct the sensitivity analysis on DNN architectures
under Setup~2.
The invariance property of our FI measure shown in 
Theorem~\ref{thm1} enables us to fairly compare
the effects of small perturbations to model parameters of different scales within or between DNNs.
Setup~2.1 compares the sensitivity
between the two DNNs, while Setup~2.2 
undertakes the comparison across trainable layers within each single DNN.

The Manhattan plots for Setup 2 on CIFAR10 are presented in Figure~\ref{setup2 CIFAR10};
results for MNIST are provided in the Supplementary Material.
The patterns on CIFAR10 under Setup 2.1 in Figure~\ref{setup2 CIFAR10}(a)-(c), with mostly smaller FIs for DensetNet121,
are quite similar to those for Setup 1 in Figure~\ref{setup1 fig}(a)-(c), indicating that
DenseNet121 is generally less sensitive 
than ResNet50 to the infinitesimal perturbations
to all network trainable parameters. 
From Figure~\ref{setup2 CIFAR10}(d)-(f) for Setup~2.2, we see stable patterns of FIs over the trainable layers for the two DNNs.
Modifying their network architectures does not appear to be necessary here.
Note that the FI value for the trainable parameters in each single network layer is theoretically dominated by that for all trainable parameters of the entire network, which
is well supported by the comparison between Figure~\ref{setup2 CIFAR10}(a)-(c) and
(d)-(f).

\subsection{Sensitivity Comparison between Training and Test Sets}

We compare the network sensitivity between training and test sets
under Setup~3.
Figure~\ref{setup3 CIFAR10} and
Table~\ref{percentile table} show the FI values for Setup 3 on CIFAR10; results for MNIST are also provided in the Supplementary Material.
In the figure and table, the test set has more slightly large FIs than the training set for both DNNs,
while FIs are generally smaller in both sets for DenseNet121.
We suggest to select a DNN model with similar sensitivity performance 
and smaller FI values on both training and test sets.
Together with the results of Sections~\ref{sec: outlier}
and~\ref{sec: sen archi}
shown in Figure~\ref{setup1 fig} (a)-(c) and Figure~\ref{setup2 CIFAR10} (a)-(c), and also with the model accuracies in Table~\ref{acc table},
DenseNet121 is preferred over ResNet50 on CIFAR10 in terms of both sensitivity and accuracy. 

\begin{table}[h]
	\caption{Percentiles of FI values for Setup 3 on CIFAR10}
	\label{percentile table}
	\centering
	\scalebox{1}{
		\begin{tabular}{cccccccc}
			\toprule
			&	\multicolumn{2}{c}{ResNet50}    & &   \multicolumn{2}{c}{DenseNet121}           \\
			\cmidrule{2-3}\cmidrule{5-6}
			Percentile &	Training     & Test    & & Training     & Test  \\
			\midrule
			75th &2.87e-3 &0.031 && 1.10e-4  & 1.83e-3\\
80th &5.38e-3 & 0.073  && 2.42e-4 & 6.57e-3\\
85th &0.010 &0.160  && 6.00e-4 & 0.025 \\
90th &0.023 & 0.343 && 1.78e-3 & 0.097 \\
95th&0.064 &  0.678 && 7.80e-3 & 0.352\\
98th & 0.177 & 0.999 && 0.037 & 0.755\\
99th & 0.316 & 1.215 && 0.099 & 0.951\\
100th  &2.160 &   3.579 && 2.533 & 2.215\\
			\bottomrule
		\end{tabular}
	}
\end{table}

\subsection{Vulnerable Region Detection}
We apply the multi-scale strategy in Setup~4 to 
detect the areas in an image that are vulnerable to
small perturbations.

For Setup 4, the test images from the two benchmark datasets with the largest FI in Setup~3 by 
DenseNet121
are illustrated in Figure~\ref{setup4 fig}.
The vulnerable areas for both images are mainly in or around the object,
and the image boundaries are generally less sensitive to perturbations.
The figure also reasonably shows that 
the vulnerable areas expand as the scale of pixel-level FI increases.

\begin{figure}[b!]
		\centering
	\includegraphics[width=0.97\linewidth]{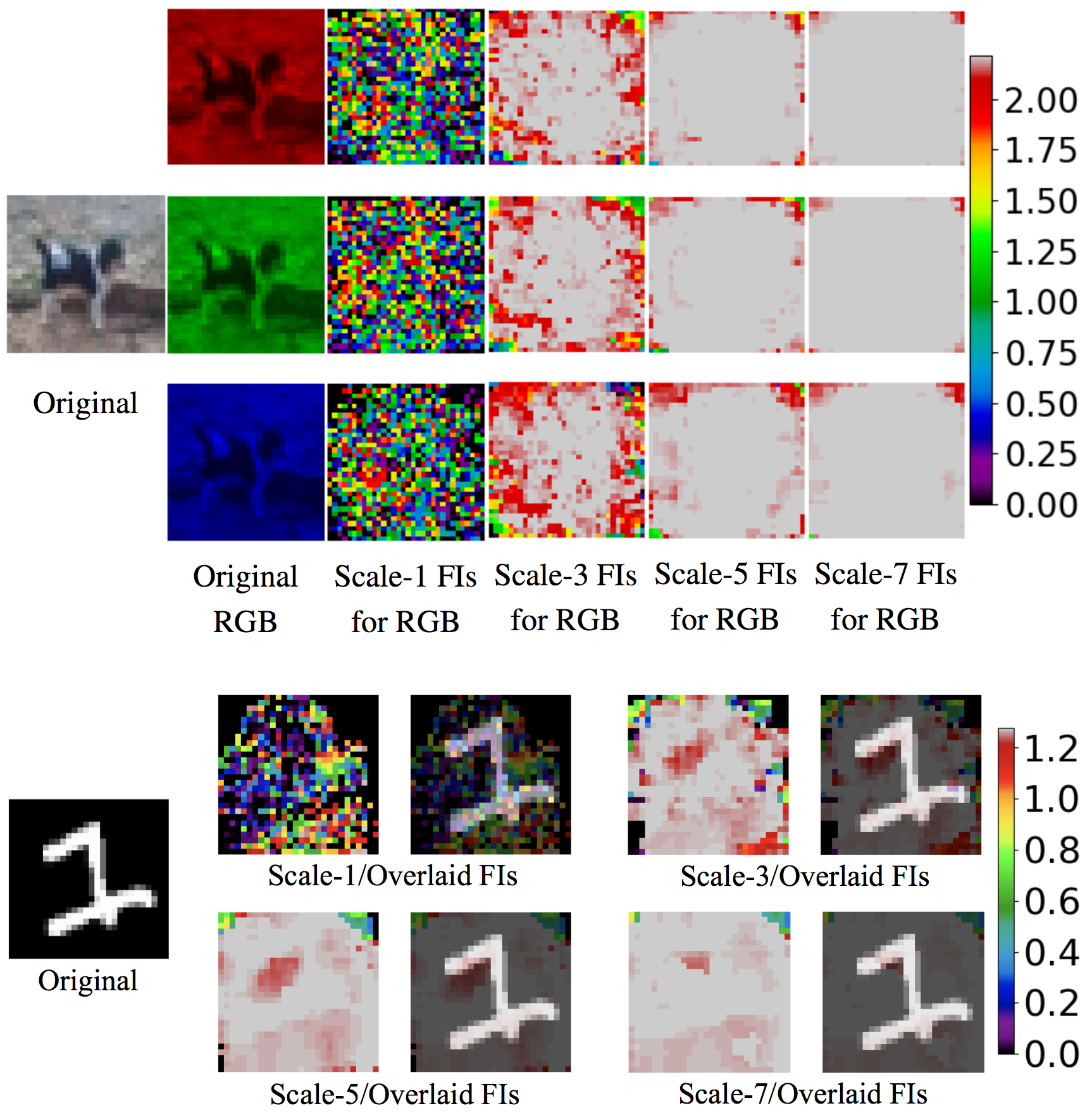}	
	\caption{Multi-scale pixel-level FI maps for Setup 4 using DenseNet121. Results are shown for the test image with the largest FI in Setup 3. The CIFAR10 test image has $\text{Setup-3 FI}=2.22$, $y_{\text{true}}=\text{dog}$, and $y_{\text{pred}}=\text{bird}$. The MNIST test image has $\text{Setup-3 FI}=1.28$, $y_{\text{true}}=1$, and $y_{\text{pred}}=7$. } 	
	\label{setup4 fig}
\end{figure}

Figure~\ref{setup0 fig} illustrates the one-pixel adversarial attacks based on pixel-wise FI maps.
The two selected test images are  correctly predicted by ResNet50 with a high probability 
and also with a large FI in Setup~3.
The pixel-wise FI map denotes the scale-1 pixel-level FI map for the MNIST image, and is the average scale-1 map over the three RGB channels for the CIFAR10 image.
For each image, the attacked pixel is the one with the largest 
value in the pixel-wise FI map.
We see that the prediction result significantly changes after slightly altering the selected pixel's value. This indicates that our FI measure is useful for discovering vulnerable locations and crafting adversarial examples.

\begin{figure}
	\centering
	\includegraphics[width=\linewidth]{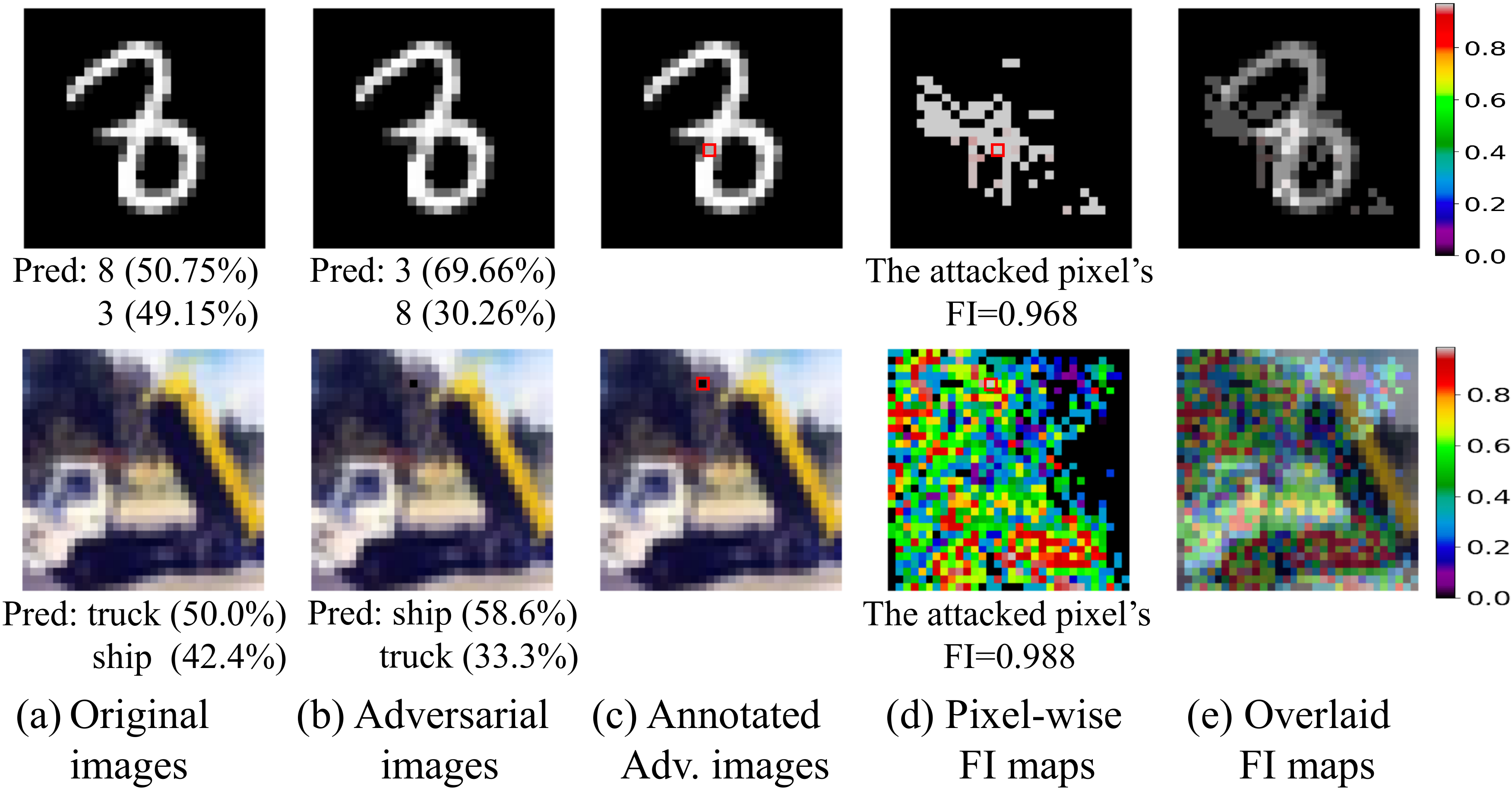}		
	\caption{One-pixel adversarial attacks on ResNet50 using pixel-wise FI maps. The original images have $y_{\text{true}}=8$ and $y_{\text{true}}=\text{truck}$, respectively. The prediction probabilities from ResNet50 are given in the parentheses. 
		The attacked pixels are framed in red.}
	\label{setup0 fig}
\end{figure}

\section{Conclusion}
In this paper, we introduced a novel perturbation manifold and its associated influence measure 
for sensitivity analysis of DNN classifiers. 
This new measure is constructed 
from a Riemannian manifold and provides the invariance property under any diffeomorphic (e.g., scaling) reparameterization of perturbations. 
This invariance property is not owned by 
the widely used measures like the Jacobian norm and Cook's local influence.
Our influence measure performs very well 
for ResNet50 and DenseNet121
trained on CIFAR10 and MNIST datasets
in the tasks of outlier detection, sensitivity comparison between network architectures and that between training and test sets, and vulnerable region detection.

\bibliographystyle{aaai}
\bibliography{Ref_robust}

\begin{thebibliography}{}

\bibitem[\protect\citeauthoryear{Abadi \bgroup et al\mbox.\egroup
  }{2016}]{Abad16}
Abadi, M.; Barham, P.; Chen, J.; Chen, Z.; Davis, A.; Dean, J.; Devin, M.;
  Ghemawat, S.; Irving, G.; Isard, M.; Kudlur, M.; Levenberg, J.; Monga, R.;
  Moore, S.; Murray, D.~G.; Steiner, B.; Tucker, P.; Vasudevan, V.; Warden, P.;
  Wicke, M.; Yu, Y.; and Zheng, X.
\newblock 2016.
\newblock Tensorflow: A system for large-scale machine learning.
\newblock In {\em 12th {USENIX} Symposium on Operating Systems Design and
  Implementation ({OSDI} 16)},  265--283.

\bibitem[\protect\citeauthoryear{Akhtar and Mian}{2018}]{Akht18}
Akhtar, N., and Mian, A.
\newblock 2018.
\newblock Threat of adversarial attacks on deep learning in computer vision: A
  survey.
\newblock {\em arXiv preprint}  arXiv:1801.00553.

\bibitem[\protect\citeauthoryear{Amari and Nagaoka}{2000}]{Amar00}
Amari, S., and Nagaoka, H.
\newblock 2000.
\newblock {\em Methods of Information Geometry}.
\newblock American Mathematical Society, Providence, RI.

\bibitem[\protect\citeauthoryear{Amari}{1985}]{Amar85}
Amari, S.
\newblock 1985.
\newblock {\em Differential-geometrical Methods in Statistics}.
\newblock Springer-Verlag, New York.

\bibitem[\protect\citeauthoryear{Bojarski \bgroup et al\mbox.\egroup
  }{2016}]{Boja16}
Bojarski, M.; Del~Testa, D.; Dworakowski, D.; Firner, B.; Flepp, B.; Goyal, P.;
  Jackel, L.~D.; Monfort, M.; Muller, U.; Zhang, J.; et~al.
\newblock 2016.
\newblock End to end learning for self-driving cars.
\newblock {\em arXiv preprint}  arXiv:1604.07316.

\bibitem[\protect\citeauthoryear{Carlini and Wagner}{2017}]{Carl17}
Carlini, N., and Wagner, D.
\newblock 2017.
\newblock Towards evaluating the robustness of neural networks.
\newblock In {\em 2017 IEEE Symposium on Security and Privacy},  39--57.

\bibitem[\protect\citeauthoryear{Cheney, Schrimpf, and
  Kreiman}{2017}]{Cheney17}
Cheney, N.; Schrimpf, M.; and Kreiman, G.
\newblock 2017.
\newblock On the robustness of convolutional neural networks to internal
  architecture and weight perturbations.
\newblock {\em arXiv preprint}  arXiv:1703.08245.

\bibitem[\protect\citeauthoryear{Cook}{1986}]{Cook86}
Cook, R.~D.
\newblock 1986.
\newblock Assessment of local influence.
\newblock {\em Journal of the Royal Statistical Society. Series B
  (Methodological)} 48(2):133--169.

\bibitem[\protect\citeauthoryear{Davis and Goadrich}{2006}]{davis2006}
Davis, J., and Goadrich, M.
\newblock 2006.
\newblock The relationship between precision-recall and roc curves.
\newblock In {\em Proceedings of the 23rd International Conference on Machine
  learning},  233--240.

\bibitem[\protect\citeauthoryear{Fawzi, Moosavi-Dezfooli, and
  Frossard}{2017}]{Fawz17}
Fawzi, A.; Moosavi-Dezfooli, S.-M.; and Frossard, P.
\newblock 2017.
\newblock The robustness of deep networks: A geometrical perspective.
\newblock {\em IEEE Signal Processing Magazine} 34(6):50--62.

\bibitem[\protect\citeauthoryear{Goodfellow \bgroup et al\mbox.\egroup
  }{2016}]{goodfellow2016deep}
Goodfellow, I.; Bengio, Y.; Courville, A.; and Bengio, Y.
\newblock 2016.
\newblock {\em Deep Learning}.
\newblock Cambridge: MIT Press.

\bibitem[\protect\citeauthoryear{Goodfellow, Shlens, and
  Szegedy}{2015}]{Good15}
Goodfellow, I.~J.; Shlens, J.; and Szegedy, C.
\newblock 2015.
\newblock Explaining and harnessing adversarial examples.
\newblock In {\em International Conference on Learning Representations}.
\newblock arXiv:1412.6572.

\bibitem[\protect\citeauthoryear{He \bgroup et al\mbox.\egroup }{2015}]{He15}
He, K.; Zhang, X.; Ren, S.; and Sun, J.
\newblock 2015.
\newblock Delving deep into rectifiers: Surpassing human-level performance on
  imagenet classification.
\newblock In {\em Proceedings of the IEEE International Conference on Computer
  Vision},  1026--1034.

\bibitem[\protect\citeauthoryear{He \bgroup et al\mbox.\egroup }{2016}]{He16}
He, K.; Zhang, X.; Ren, S.; and Sun, J.
\newblock 2016.
\newblock Deep residual learning for image recognition.
\newblock In {\em Proceedings of the IEEE Conference on Computer Vision and
  Pattern Recognition},  770--778.

\bibitem[\protect\citeauthoryear{Hein and Andriushchenko}{2017}]{Hein17}
Hein, M., and Andriushchenko, M.
\newblock 2017.
\newblock Formal guarantees on the robustness of a classifier against
  adversarial manipulation.
\newblock In {\em Advances in Neural Information Processing Systems},
  2266--2276.

\bibitem[\protect\citeauthoryear{Huang \bgroup et al\mbox.\egroup
  }{2017}]{Huan17}
Huang, G.; Liu, Z.; Weinberger, K.~Q.; and van~der Maaten, L.
\newblock 2017.
\newblock Densely connected convolutional networks.
\newblock In {\em Proceedings of the IEEE Conference on Computer Vision and
  Pattern Recognition},  4700--4708.

\bibitem[\protect\citeauthoryear{Krizhevsky, Sutskever, and
  Hinton}{2012}]{Kriz12}
Krizhevsky, A.; Sutskever, I.; and Hinton, G.~E.
\newblock 2012.
\newblock Imagenet classification with deep convolutional neural networks.
\newblock In {\em Advances in Neural Information Processing Systems},
  1097--1105.

\bibitem[\protect\citeauthoryear{Moosavi-Dezfooli, Fawzi, and
  Frossard}{2016}]{Moos16}
Moosavi-Dezfooli, S.-M.; Fawzi, A.; and Frossard, P.
\newblock 2016.
\newblock Deepfool: a simple and accurate method to fool deep neural networks.
\newblock In {\em Proceedings of the IEEE Conference on Computer Vision and
  Pattern Recognition},  2574--2582.

\bibitem[\protect\citeauthoryear{Novak \bgroup et al\mbox.\egroup
  }{2018}]{Nova18}
Novak, R.; Bahri, Y.; Abolafia, D.~A.; Pennington, J.; and Sohl-Dickstein, J.
\newblock 2018.
\newblock Sensitivity and generalization in neural networks: an empirical
  study.
\newblock In {\em International Conference on Learning Representations}.
\newblock arXiv:1802.08760.

\bibitem[\protect\citeauthoryear{Paszke \bgroup et al\mbox.\egroup
  }{2017}]{Pasz17}
Paszke, A.; Gross, S.; Chintala, S.; Chanan, G.; Yang, E.; DeVito, Z.; Lin, Z.;
  Desmaison, A.; Antiga, L.; and Lerer, A.
\newblock 2017.
\newblock Automatic differentiation in pytorch.
\newblock In {\em NIPS 2017 Autodiff Workshop}.

\bibitem[\protect\citeauthoryear{Russakovsky \bgroup et al\mbox.\egroup
  }{2015}]{Russ15}
Russakovsky, O.; Deng, J.; Su, H.; Krause, J.; Satheesh, S.; Ma, S.; Huang, Z.;
  Karpathy, A.; Khosla, A.; Bernstein, M.; et~al.
\newblock 2015.
\newblock Imagenet large scale visual recognition challenge.
\newblock {\em International Journal of Computer Vision} 115(3):211--252.

\bibitem[\protect\citeauthoryear{Sharif \bgroup et al\mbox.\egroup
  }{2017}]{Shar17}
Sharif, M.; Bhagavatula, S.; Bauer, L.; and Reiter, M.~K.
\newblock 2017.
\newblock Adversarial generative nets: Neural network attacks on
  state-of-the-art face recognition.
\newblock {\em arXiv preprint}  arXiv:1801.00349.

\bibitem[\protect\citeauthoryear{Su, Vargas, and Kouichi}{2017}]{Su17}
Su, J.; Vargas, D.~V.; and Kouichi, S.
\newblock 2017.
\newblock One pixel attack for fooling deep neural networks.
\newblock {\em arXiv preprint}  arXiv:1710.08864.

\bibitem[\protect\citeauthoryear{Szegedy \bgroup et al\mbox.\egroup
  }{2013}]{Szeg13}
Szegedy, C.; Zaremba, W.; Sutskever, I.; Bruna, J.; Erhan, D.; Goodfellow, I.;
  and Fergus, R.
\newblock 2013.
\newblock Intriguing properties of neural networks.
\newblock In {\em International Conference on Learning Representations}.
\newblock arXiv:1312.6199.

\bibitem[\protect\citeauthoryear{Weng \bgroup et al\mbox.\egroup
  }{2018}]{Weng18}
Weng, T.-W.; Zhang, H.; Chen, P.-Y.; Yi, J.; Su, D.; Gao, Y.; Hsieh, C.-J.; and
  Daniel, L.
\newblock 2018.
\newblock Evaluating the robustness of neural networks: An extreme value theory
  approach.
\newblock In {\em International Conference on Learning Representations}.
\newblock arXiv:1801.10578.

\bibitem[\protect\citeauthoryear{Zhu \bgroup et al\mbox.\egroup }{2007}]{Zhu07}
Zhu, H.; Ibrahim, J.~G.; Lee, S.; Zhang, H.; et~al.
\newblock 2007.
\newblock Perturbation selection and influence measures in local influence
  analysis.
\newblock {\em The Annals of Statistics} 35(6):2565--2588.

\bibitem[\protect\citeauthoryear{Zhu, Ibrahim, and Tang}{2011}]{Zhu11}
Zhu, H.; Ibrahim, J.~G.; and Tang, N.
\newblock 2011.
\newblock Bayesian influence analysis: a geometric approach.
\newblock {\em Biometrika} 98(2):307--323.

\end{thebibliography}

\end{document}